\documentclass{article}

\usepackage{microtype}
\usepackage{graphicx}
\usepackage{subfigure}
\usepackage{booktabs} 

\usepackage{hyperref}

\usepackage[accepted]{icml2024}

\usepackage{amsmath}
\usepackage{amssymb}
\usepackage{mathtools}
\usepackage{amsthm}

\usepackage[capitalize,noabbrev]{cleveref}

\theoremstyle{plain}
\newtheorem{theorem}{Theorem}[section]
\newtheorem{proposition}[theorem]{Proposition}
\newtheorem{lemma}[theorem]{Lemma}
\newtheorem{corollary}[theorem]{Corollary}
\theoremstyle{definition}
\newtheorem{definition}[theorem]{Definition}

\theoremstyle{remark}
\newtheorem*{remark}{Remark}
\newtheorem{example}[theorem]{Example}

\usepackage[textsize=tiny]{todonotes}

\icmltitlerunning{The Illusion of State in State-Space Models}

\usepackage{amsmath}
\usepackage{amssymb}
\usepackage{amsthm}
\usepackage{thm-restate}
\usepackage{mathtools}
\usepackage{bbm}

\usepackage{cleveref}
\usepackage{multirow}
\usepackage{xspace}
\usepackage{booktabs}
\usepackage{subcaption}
\usepackage{xcolor}
\usepackage{paralist}
\usepackage{graphicx}

\usepackage{etoolbox}

\makeatletter

\patchcmd{\NAT@test}{\else \NAT@nm}{\else \NAT@nmfmt{\NAT@nm}}{}{}

\DeclareRobustCommand\citepos
  {\begingroup
   \let\NAT@nmfmt\NAT@posfmt
   \NAT@swafalse\let\NAT@ctype\z@\NAT@partrue
   \@ifstar{\NAT@fulltrue\NAT@citetp}{\NAT@fullfalse\NAT@citetp}}

\let\NAT@orig@nmfmt\NAT@nmfmt
\def\NAT@posfmt#1{\NAT@orig@nmfmt{#1's}}

\makeatother

\usepackage{pifont}  
\DeclarePairedDelimiter\abs{\lvert}{\rvert}%
\DeclarePairedDelimiter\norm{\lVert}{\rVert}%

\makeatletter
\let\oldabs\abs
\def\abs{\@ifstar{\oldabs}{\oldabs*}}
\let\oldnorm\norm
\def\norm{\@ifstar{\oldnorm}{\oldnorm*}}
\makeatother

\def\Snospace~{\S{}}

\newcommand{\AC}{\mathsf{AC}}
\newcommand{\NC}{\mathsf{NC}}
\newcommand{\TC}{\mathsf{TC}}

\newcommand{\poly}{\mathrm{poly}}

\newcommand{\sgn}{\mathrm{sgn}}

\renewcommand{\L}{\mathsf{L}}
\renewcommand{\P}{\mathsf{P}}

\newcommand{\D}{\mathbb{D}}
\newcommand{\R}{\mathbb{R}}
\newcommand{\Z}{\mathbb{Z}}

\DeclareMathOperator{\diag}{diag}

\usepackage{chessboard}
\usepackage{skak}
\setchessboard{showmover=false}
\definecolor{move1}{rgb}{0, 0.3, 0.5}
\definecolor{move2}{rgb}{0.2, 0.5, 0.7}
\definecolor{move3}{rgb}{0.4, 0.7, 0.9}
\definecolor{move1red}{rgb}{0.5, 0.3, 0}
\definecolor{move2red}{rgb}{0.7, 0.5, 0.2}
\definecolor{move3red}{rgb}{0.9, 0.7, 0.4}
\definecolor{movegray}{rgb}{0.3, 0.3, 0.3}
\newcommand\m[1]{\textsf{#1}}

\DeclareFixedFont{\ttb}{T1}{txtt}{bx}{n}{12} 
\DeclareFixedFont{\ttm}{T1}{txtt}{m}{n}{12}  

\definecolor{deepblue}{rgb}{0,0,0.5}
\definecolor{deepred}{rgb}{0.6,0,0}
\definecolor{deepgreen}{rgb}{0,0.5,0}
\definecolor{deepgray}{rgb}{0.5, 0.5, 0.5}

\usepackage{pythonhighlight}

\newcommand{\cast}{\mathsf{cast}}

\usepackage{tikz}
\usetikzlibrary{shapes.geometric,calc,positioning}

\colorlet{new}{blue}
\definecolor{lightred}{RGB}{255, 204, 203}
\definecolor{lightpurple}{RGB}{203, 195, 227}

\newcommand{\barA}{\bar{\mathbf A}}
\newcommand{\barB}{\bar{\mathbf B}}

\usepackage{pict2e}

\makeatletter
\newcommand{\bigcomp}{%
  \DOTSB
  \mathop{\vphantom{\sum}\mathpalette\bigcomp@\relax}%
  \slimits@
}
\newcommand{\bigcomp@}[2]{%
  \begingroup\m@th
  \sbox\z@{$#1\sum$}%
  \setlength{\unitlength}{0.9\dimexpr\ht\z@+\dp\z@}%
  \vcenter{\hbox{%
    \begin{picture}(1,1)
    \bigcomp@linethickness{#1}
    \put(0.5,0.5){\circle{1}}
    \end{picture}%
  }}%
  \endgroup
}
\newcommand{\bigcomp@linethickness}[1]{%
  \linethickness{%
      \ifx#1\displaystyle 2\fontdimen8\textfont\else
      \ifx#1\textstyle 1.65\fontdimen8\textfont\else
      \ifx#1\scriptstyle 1.65\fontdimen8\scriptfont\else
      1.65\fontdimen8\scriptscriptfont\fi\fi\fi 3
  }%
}

\makeatother

\newcommand\softplus{\mathrm{softplus}}

\begin{document}

\twocolumn[
\icmltitle{The Illusion of State in State-Space Models}




\begin{icmlauthorlist}
\icmlauthor{William Merrill}{nyu}
\icmlauthor{Jackson Petty}{nyu}
\icmlauthor{Ashish Sabharwal}{ai2}
\end{icmlauthorlist}

\icmlaffiliation{nyu}{New York University}
\icmlaffiliation{ai2}{Allen Institute for AI}

\icmlcorrespondingauthor{William Merrill}{willm@nyu.edu}
\icmlcorrespondingauthor{Jackson Petty}{petty@nyu.edu}
\icmlcorrespondingauthor{Ashish Sabharwal}{ashishs@allenai.org}

\icmlkeywords{Machine Learning, ICML}

\vskip 0.3in
]



\printAffiliationsAndNotice{}  


\begin{abstract}
    State-space models (SSMs) have emerged as a potential alternative to transformers.
    One theoretical weakness of transformers is that they cannot express certain kinds of sequential computation and state tracking \citep{merrill-sabharwal-2023-parallelism}, which SSMs are explicitly designed to address via their close architectural similarity to recurrent neural networks. \emph{But do SSMs truly have an advantage (over transformers) in expressive power for state tracking?} Surprisingly, the answer is no. Our analysis reveals that the expressive power of S4, Mamba, and related SSMs is limited very similarly to transformers (within $\TC^0$), meaning
    these SSMs cannot solve simple state-tracking problems like permutation composition and consequently are provably unable to accurately track chess moves with certain notation, evaluate code, or track entities in a long narrative. To supplement our formal analysis, we report experiments showing that S4 and Mamba indeed struggle with state tracking. Thus, despite their recurrent formulation, the ``state'' in common SSMs is an illusion: S4, Mamba, and related models have similar expressiveness limitations to non-recurrent models like transformers, which may fundamentally limit their ability to solve real-world state-tracking problems. Moreover, we show that only a minimal change allows SSMs to express and learn state tracking, motivating the development of new, more expressive SSM architectures.
\end{abstract}

\section{Introduction}

\begin{figure}[t]
    \centering
    \resizebox{.38\textwidth}{!}{
        \chessboard[
            setpieces={qa8, qb8, rc8, qd8, qe8, Pa2, Pb2, Ka1, kh1},
            addpgf={
                \tikz[overlay]\draw[move1red,line width=0.1em,->](c8) to[out=30,in=30] (c6);
                \tikz[overlay]\draw[move2red,line width=0.1em,->](c6)--(a6);
                \tikz[overlay]\draw[move3red,line width=0.1em,->](a6) to[out=150,in=210] (a8);
                \tikz[overlay]\draw[move1,line width=0.1em,->](a8)--(a7);
                \tikz[overlay]\draw[move2,line width=0.1em,->](a7)--(c7);
                \tikz[overlay]\draw[move3,line width=0.1em,->](c7)--(c8);
                \tikz[overlay]\draw[movegray,line width=0.1em,<->](a1)--(b1);
            },
        ]
    }

    \noindent\fbox{\parbox{0.95\columnwidth}{%
        \small
        \pyth{x = [0, 0, 1, 0, 0]}\\
        \pyth{x[1], x[3] = x[3], x[1]  # Swap 1, 3}
    }}
    \vspace{0.5em}

    \noindent\fbox{\parbox{0.95\columnwidth}{%
        \emph{Alice, Bob, Carl, Dan, and Emma each have a coin. All are dimes except Carl's.
        Alice and Carl trade coins.}
    }}
    
    \caption{
    We prove that SSMs, like transformers, cannot solve inherently sequential problems like permutation composition ($S_5$), which lies at the heart of state-tracking problems like tracking chess moves in source-target notation  (see~\Cref{sec:chess}), evaluating Python code, or entity tracking. Thus, SSMs cannot, in general, solve these problems either.
    \begin{tabular}{c} \\
         \hspace{0.7em}\textbf{Code:} \url{http://jpetty.org/ssm-illusion}
    \end{tabular}
    }
    \label{fig:s5-chess-code}
\end{figure}

Recent theoretical work has shown that transformer architecture based models are incapable of expressing inherently sequential computation \citep{merrill-sabharwal-2023-parallelism}. These results reveal a surprising limitation of transformers: they cannot express simple kinds of \emph{state tracking} problems, such as composing sequences of permutations, which even simple recurrent neural networks (RNNs) can naturally express.
In a different line of work, state space model (SSM) architectures \citep{gu2021combining,gu2022efficiently,fu2023hungry,gu2023mamba,wang2024mambabyte} have been introduced as an alternative to transformers, with the goal of achieving RNN-like expressive power for handling problems that are naturally stateful and sequential \citep{gu2021combining,gu2022blog}.
\emph{But does the seemingly stateful design of SSMs truly enable them to solve sequential and state-tracking problems that transformers cannot?}
If so, this would be a promising property of SSMs because state tracking is at the heart of large language model (LLM) capabilities such as tracking entities in a narrative \citep{heim1983file,kim-schuster-2023-entity}, playing chess under certain notation\footnote{The hardness of chess state tracking holds with (source, target) notation, but standard notation may make state tracking easier.}, or evaluating code. This would motivate further research on SSM architectures and their deployment in the next generation of LLMs.


In this work, we show that the apparent stateful design of SSMs is an \emph{illusion} as far as their expressive power is concerned. In contrast to the suggestion by \citet{gu2021combining,gu2022blog} (and, perhaps, a broader belief in the community) that SSMs have expressive power for state tracking similar to RNNs, we prove theoretically that linear and Mamba-style SSMs, like transformers, cannot express inherently sequential problems, including state-tracking problems like composing permutations that RNNs can easily express. Further, our experiments confirm this prediction: both transformers and these SSMs cannot learn to compose permutations with a fixed number of layers, whereas RNNs can compose permutations with just a single layer.
Our results imply that arguments that current SSMs have an advantage over transformers due to being ``more recurrent'' or capable of tracking state are misguided. In fact, the SSM architectures we consider are just as theoretically unequipped for state tracking and recurrent computation as transformers are.

We first establish the theoretical weakness of linear SSMs and near generalizations by proving they are in the complexity class $\L$-uniform $\TC^0$, which has been previously shown for transformers \citep{merrill-sabharwal-2023-parallelism}. This implies these SSMs cannot solve inherently sequential problems (formally, problems that are $\NC^1$-hard), including state-tracking problems like permutation composition \citep{liu2023transformers}. Permutation composition is a fundamental problem at the heart of many real-world state-tracking problems such as playing chess, evaluating code, or tracking entities in a narrative (\Cref{fig:s5-chess-code}), implying solutions to these problems, too, cannot be expressed by SSMs, at least in the worst case.

At first glance, our results may appear to contradict \citet{gu2021combining}'s claim that linear SSMs can simulate general recurrent models, which can express permutation composition. But the contradiction is resolved by a difference in assumptions: \citet{gu2021combining} relied on \emph{infinite depth}
(number of layers) to show that SSMs could simulate RNNs. We, on the other hand, analyze the realistic setting with a bounded number of layers, under which we find that SSMs cannot simulate the recurrent state of an RNN and, in fact, suffer from similar limitations as transformers for state tracking.

Empirically, we find that S4 \citep{gu2022efficiently} and S6 \citep{gu2023mamba} SSMs, as well as transformers, do \emph{not} learn to solve the permutation composition state-tracking problem with a fixed number of layers, while simple RNNs can do so with just one layer. This provides empirical support for our theoretical separation in expressive power for state tracking between SSMs and true recurrent models. We also find that both transformers and SSMs struggle compared to RNNs on state-tracking problems less complex than permutation composition where it is not known whether they can express a solution.
Thus, in practice, SSMs may struggle not just on the hardest state-tracking problems like permutation composition but also on easier variants.

\begin{figure}
    \centering
    \begin{tikzpicture}
    \node[
          above,ellipse,draw,
          minimum height=8em,
          minimum width=18em,
          fill=lightred,
         ] (nc1) {};
    \node[
          above,ellipse,draw,
          minimum height=6em,
          minimum width=14em,
          fill=lightpurple,
         ] (tc0) {};

    \node[below left=0.25cm and 0.7cm of nc1.north,font=\footnotesize] {\textit{Chess}};
    \node[below left=0.6cm and 2.0cm of nc1.north,font=\footnotesize] {$S_5$};
    \node[below right=0.25cm and 0.7cm of nc1.north,font=\footnotesize] {\textit{Code}};
    \node[below right=0.6cm and 1.6cm of nc1.north,font=\footnotesize] {\textit{Entities}};

    \node[
          left = 1.2cm of tc0.center,
          font=\footnotesize,
          anchor=center,
         ] {SSMs};
    \node[
          right = 1cm of tc0.center,
          font=\footnotesize,
          anchor=center,
         ] {Transformers};

    \path
        (tc0.north) node[below] {$\TC^0$}
        (nc1.north) node[below] {$\NC^1$};
    \end{tikzpicture}
    \caption{Complexity hierarchy within $\NC^1$. Transformers can only recognize languages within $\TC^0$ \citep{merrill-sabharwal-2023-parallelism}, and we show the same for SSMs (\Cref{thm:non-gated,thm:diagonal}). Thus, both architectures cannot express the ``hard state tracking'' captured by $\NC^1$-complete problems like $S_5$, which \emph{can} be straightforwardly expressed by RNNs. The figure assumes the widely held conjecture $\TC^0 \neq \NC^1$.}
    \label{fig:hierarchy}
\end{figure}

Finally, we consider a minimal extension of a linear SSM which makes the transition matrix input dependent, similar to Liquid S4 \cite{hasani2023liquid}. We show that this extension has sufficient expressive power for state tracking and permutation composition. Empirically, we show that our implementation of this extension learns to solve permutation composition with a single layer, just like an RNN, while being similarly parallelizable to other SSMs.
It is an open question whether such SSM architectures with greater expressivity for state tracking are practically viable for large-scale language modeling.

\section{Background}

We first present the SSM architectures we will analyze (\Cref{sec:ssms}).
Our analysis of the state tracking capabilities of SSMs borrows deeply from the circuit complexity and algebraic formal language theory literature. We thus review how circuit complexity can be used to analyze the power of neural networks (\Cref{sec:circuits}) and how state-tracking problems can be captured algebraically and analyzed within the circuit complexity framework (\Cref{sec:state-tracking}).

\subsection{Architecture of State-Space Models} \label{sec:ssms}

SSMs are a neural network architecture for processing sequences similar in design to RNNs or linear dynamical systems.
SSMs have been suggested to have two potential advantages compared to transformers owing to their recurrent formulation: faster inference and, possibly, the ability to better express inherently sequential or stateful problems \citep{gu2021combining,gu2022blog}.
Several architectural variants of SSMs have been proposed, including S4 \citep{gu2022efficiently} and Mamba \citep{gu2023mamba}.
Recently, SSMs have been shown to achieve strong empirical performance compared to transformers in certain settings, particularly those involving a long context \citep{gu2023mamba,wang2024mambabyte}.

SSMs consist of \emph{SSM layers}, which can be thought of as simplified RNN layers. We define a \emph{generalized linear SSM layer} that encapsulates both S4 \citep{gu2022efficiently} and the S6 layer used by Mamba \citep{gu2023mamba} as special cases.

\begin{definition}[Generalized linear SSM layer]
    \label{def:generalized-linear-ssm}
    Given a sequence\footnote{In practice, an SSM is often applied elementwise ($k=1$) on each feature $1, \ldots, m$ of hidden state $\mathbf x_1, \ldots, \mathbf x_n \in \mathbb R^m$.} $\mathbf x_1, \ldots, \mathbf x_n \in \mathbb R^k$, the \emph{recurrent form} of a linear SSM layer defines a new sequence of states $\mathbf h_1, \ldots, \mathbf h_n \in \mathbb R^d$ using projections $\barA_i \in \mathbb R^{d \times d}$ and $\barB_i \in \mathbb R^{d \times k}$, which can themselves depend on $\mathbf x_i$. For each $1 \leq i \leq n$,
    \begin{equation}
        \label{eqn:recurrent}
        \mathbf h_i = \barA_i \mathbf h_{i-1} + \barB_i \mathbf x_i .
        \tag{Recurrent form}
    \end{equation}
    The \emph{convolutional form} of the SSM layer defines the same\footnote{The two forms express the same function over $\mathbb R$ or any other distributive datatype. Over floating points (\Cref{sec:datatype}), they are not guaranteed to be the same, but we must assume the error is negligible for them to be well-defined and usable in practice.} $\mathbf h_1, \ldots, \mathbf h_n$ computed differently as a summation:\footnote{We define the cumulative product to unroll from greatest to least, e.g., $\prod_1^3 \mathbf A_i = \mathbf A_3 \mathbf A_2 \mathbf A_1$. The order is important due to the non-commutativity of matrix multiplication.}
    \begin{equation}
        \label{eqn:convolutional}
        \mathbf h_i = \sum_{j=1}^i \left( \prod_{k=j+1}^i \barA_k \right) \barB_j \mathbf x_j .
        \tag{Convolutional form}
    \end{equation}
    The layer outputs $\mathbf y_i = \mathbf C_i \mathbf h_i + \mathbf D_i \mathbf x_i \in \mathbb R^k$, where $\mathbf C_i \in \mathbb R^{k \times d}$ and $\mathbf D_i \in \mathbb R^{k \times k}$ depend on $\mathbf x_i$.
\end{definition}

Two common cases of this layer are when $\barA_i$ does not depend on the input (``non-gated''; \Cref{sec:non-gated}) and when $\barA_i$ is diagonal (\Cref{sec:diagonal}). In both of these cases, we will show that the SSM can be simulated in $\TC^0$.

A \textbf{generalized linear SSM} is made up of multiple such layers, with a linear projection and a non-linearity applied after every layer \citep{rush2022s4}. Layer-norm can also be applied, either before or after the layer.

\paragraph{Practical Details.}
In S4 and related SSMs, \Cref{def:generalized-linear-ssm} is applied elementwise ($k=1$) across all $m$ elements of the previous layer output \citep{gu2022efficiently}. In practice, the weight matrix initialization is crucial for training.
Our expressivity results (\Cref{thm:non-gated,thm:diagonal}) apply for any generalized linear SSM (including S4 and S6), independent of initialization.
In contrast to S4 and S6, H3 \citep{fu2023hungry} does not meet \Cref{def:generalized-linear-ssm} because the context is not represented by a single vector. Rather, it resembles a transformer with SSM components.

\subsection{Numeric Datatype}
\label{sec:datatype}

Circuit-complexity analysis of neural networks depends to some degree on low-level details about arithmetic and the underlying datatype $\D$ used in the network's computation graph. We can think of $\D$ as parameterized by the number of bits available to represent a number in $\D$. For instance, non-negative integers in $[0, 2^p]$ use $p$ bits, signed integers in $[-2^p, 2^p]$ use $p+1$ bits, FP16 uses 16 bits, etc.

Our main results (\Cref{thm:non-gated,thm:diagonal}) will go through for any datatype $\D$ for which the following operations are efficiently parallel-computable (i.e., are in the complexity class $\L$-uniform $\TC^0$, to be defined shortly in \cref{sec:circuits}):
\begin{compactenum}
    \item Iterated addition, i.e., summing $n$ numbers in $\D$
    \item Iterated product, i.e., multiplying $n$ numbers in $\D$
    \item Matrix powering, i.e., computing the $n$-th power of a fixed-size $d \times d$ matrix over $\D$
\end{compactenum}

When $\D$ is any finite-precision datatype, i.e., has a fixed number of bits available (e.g., 16 or 64), then these operations are easily seen to be in $\L$-uniform $\TC^0$. As \citet{merrill2023logic} argue, however, finite-precision datatypes severely limit the expressivity of neural architectures from a formal perspective (e.g., finite-precision transformers cannot represent uniform attention), motivating the use of parameterized datatypes that can (approximately) represent any number with a sufficiently large parameter. Interestingly, when $\D$ is the datatype of $n$-bit integers, all of the above operations are known to be in $\L$-uniform $\TC^0$ \cite{hesse2001division,mereghetti2000threshold}. Realistically, however, neural model implementations use floating point numbers with much fewer than $n$ bits. Following \citet{merrill2023logic}, we use the \textbf{log-precision floating point} model, i.e., $c \log n$ bit floats where $c$ is some fixed constant (see \Cref{app:arithmetic} for a formal definition). \citet{merrill-sabharwal-2023-parallelism} showed that iterated addition over log-precision floats is in $\L$-uniform $\TC^0$. We extend the arguments of \citet{hesse2001division} and \citet{mereghetti2000threshold} to show that iterated product and matrix powering over log-precision floats are also in $\L$-uniform $\TC^0$ (see \Cref{app:arithmetic}).

\subsection{Limits of Transformers via Circuit Complexity} \label{sec:circuits}

A line of recent work has used circuit complexity and logic formalisms to identify expressivity limitations of transformers on reasoning problems \cite{angluin-2023-BRASP,merrill-sabharwal-2023-parallelism,liu2023transformers,chiang2023tighter,merrill2023logic,hao2022ac0}; see \citealp{strobl2023transformers} for a survey. In particular, \citet{merrill-sabharwal-2023-parallelism} showed transformers can only solve problems in the complexity class $\TC^0$, which is the set of problems that can be recognized by constant-depth, polynomial-size threshold circuit families. Such circuits, in addition to having standard AND, OR, and NOT gates (of arbitrary fan-in), can also use threshold gates that output $1$ iff at least $k$ of the inputs are $1$, where $k$ is a parameter of the gate. Informally, $\TC^0$ can be thought of as the class of problems that can be solved with extremely parallel (constant-depth) computation.\footnote{We use $\TC^0$ to mean $\L$-uniform $\TC^0$, meaning the circuit family is constructible by a Turing machine that runs in space logarithmic in the size of the input \citep[cf.][]{merrill-sabharwal-2023-parallelism,strobl2023transformers}.
We believe our results could be extended from $\L$-uniform $\TC^0$ to $\mathsf{DLOGTIME}$-uniform $\TC^0$ using techniques similar to \citet{merrill2023logic} for composing $\TC^0$ circuits in a way that preserves $\mathsf{DLOGTIME}$ uniformity.
}

Problems outside $\TC^0$, corresponding to problems that are inherently sequential and thus cannot be parallelized, cannot be solved by transformers. No problems in polynomial time are known unconditionally to be outside $\TC^0$, but unless the widely held conjecture that $\TC^0 \neq \NC^1$ is false, many simple $\NC^1$-hard problems are outside $\TC^0$. In particular, this includes simulating finite automata ($\NC^1$-complete), evaluating boolean formulas ($\NC^1$-complete), determining graph connectivity ($\L$-complete), and solving linear equations ($\P$-complete). These problems have already been shown to be inexpressible by transformers \citep{merrill-sabharwal-2023-parallelism}. By showing that SSMs can be simulated in $\TC^0$, we will establish that they also cannot be solved by SSMs.

\section{State Tracking} \label{sec:state-tracking}

Informally, a state-tracking problem is a problem where the text specifies some sequence of updates to the state of the world, and the goal of the problem is to determine what the world state is after the updates have been applied in sequence.
The circuit complexity view on the power of neural networks can be combined with other insights from algebraic formal language theory to analyze the kinds of state tracking that SSMs can express.
In particular, this theory reveals which kinds of state-tracking problems are (likely) not in $\TC^0$. This will, in turn, allow us to find examples of hard state tracking that models like SSMs cannot express.

\subsection{State Tracking as a Monoid Word Problem} \label{sec:word-problems}

 From the perspective of algebraic formal language theory, state tracking over a finite world can be captured as a \emph{word problem} on a \emph{finite monoid} \citep{liu2023transformers}.\footnote{We consider finite monoids for simplicity, but the approach may be extendable to infinite (e.g., finitely generated) monoids.}
Different updates to the world become different elements in the monoid, and resolving the final world state after all the updates have been applied is equivalent to computing the product of a sequence of elements (also called a ``word'').

\begin{definition}[Word problem]
    Let $M$ be a finite set, and $(M,\cdot)$ a finite monoid (i.e., $M$ with identity and associative multiplication). The word problem for $M$ is 
    to reduce sequences in $M^*$ under multiplication; that is, send $m_0 m_1 \cdots m_k$ to $m_0 \cdot m_1 \cdot \ldots \cdot m_k \in M$. Solving the word problem requires reducing sequences of arbitrary length.
\end{definition}

\begin{example} \label{ex:parity}
Consider the monoid $\{0, 1\}$ where $\cdot$ is addition modulo 2. The word problem is to compute the parity of a string, e.g., $0011 \mapsto 0$. From a state-tracking perspective, this monoid captures a world with a single light switch. Identity $0$ corresponds to no action, and $1$ flips the switch.
\end{example}



Modeling state tracking with word problems lets us draw connections between circuit complexity and algebra to understand which word problems are hard to solve. \citet{krohn1965algebraic} established that not all word problems are created equal: some, like \Cref{ex:parity}, are in $\TC^0$, while others are $\NC^1$-complete, requiring recurrent processing to solve \citep{immerman1989complexity,barrington1989bounded}. Because we will show SSMs can be simulated in $\TC^0$, it follows that $\NC^1$-complete state-tracking problems cannot be expressed by SSMs (cf.~\Cref{fig:hierarchy}).

Whether or not a word problem is $\NC^1$-complete depends on the algebraic structure of the underlying monoid. \citet{barrington1989bounded} showed that the word problem of every finite non-solvable\footnote{We focus on word problems on groups, which are monoids with inverses. Formally, a group $G$ is solvable exactly when there is a series of subgroups $1 = G_0 < G_1 < \cdots < G_k = G$ such that $G_{i-1}$ is normal in $G_i$ and $G_i/G_{i-1}$ is abelian.} group is  $\NC^1$-complete. That non-solvable groups have $\NC^1$-complete word problems is notable because of the ubiquity with which non-solvable groups show up in tasks involving state tracking. The canonical example of an $\NC^1$-complete word problem is that of $S_5$, the symmetric group on five elements that encodes the permutations over five objects. As an immediate instantiation of this, consider a document describing a sequence of transpositions: \emph{``swap ball 1 and 3, swap ball 3 and 5, swap ball 4 and 2, ...''}.
Being able to answer the question \emph{``where does ball 5 end up?''} for all possible swap sequences requires solving the $S_5$ word problem.\footnote{W.l.o.g., any permutation can be factored into a sequence of transpositions, or swaps.} 
Beyond permutations, \Cref{fig:s5-chess-code} shows how many natural state-tracking problems like tracking chess moves, evaluating code, or tracking entities also encode the structure of $S_5$, meaning these state-tracking problems also cannot be expressed by a model in $\TC^0$. Rather, in order to solve these problems, the depth of the model would have to be expanded to accommodate longer inputs.

Although the $S_5$ word problem is canonical, in this paper we will consider the word problem on a closely related group $A_5$: the \emph{alternating} group on five elements. We do this for simplicity: $A_5$ is a subgroup of $S_5$ containing only even permutations, and is the smallest non-solvable subgroup. We will compare the word problem on $A_5$ to two other baseline groups: $A_4 \times \mathbb{Z}_5$, a non-abelian but solvable group; and $\mathbb{Z}_{60}$, an abelian group encoding mod-$60$ addition. We choose these groups as points of comparison because they all have $60$ distinct elements, meaning that the difficulty in learning their word problems will come only from the complexity of learning the group multiplication operation.

\subsection{Encoding $S_5$ in Chess State Tracking} \label{sec:chess}


\Cref{fig:s5-chess-code} already gives some intuition into how state-tracking problems encode $S_5$. Out of these examples, the most intricated case is chess. We now give a proper reduction from $S_5$ to tracking chess moves, showing formally that not just $S_5$, but chess state tracking as well, is $\NC^1$-complete.
We define the chess state-tracking problem as follows:

\begin{compactitem}
    \item \textbf{Input:} A \textbf{chessboard state} and \textbf{sequence of chess moves}, where each move is written in UCI notation as a tuple (source square, target square). This differs from the standard SAN notation that represents other information like piece type \citep{Toshniwal2021ChessAA}.
    \item \textbf{Output:} The resulting board state after starting in the initial board state and applying the sequence of moves one after another, ignoring draws. If any move is illegal given the previous board state, a null state is returned.
\end{compactitem} 

We show that $S_5$ can be reduced to chess state tracking, establishing its $\NC^1$-completeness:

\begin{proposition}\label{prop:chess}
    $S_5$ can be reduced to chess state tracking in UCI notation via $\NC^0$ reductions.
\end{proposition}

\begin{proof}
    Without loss of generality, we consider the variant of $S_5$ where the output is true if and only if the original first element returns to the first position after the given sequence of permutations has been applied.
    
    The idea, as illustrated in \Cref{fig:s5-chess-code}, is to map each element of $S_5$ to a fixed sequence of chess moves that permutes five pieces accordingly on the chessboard.
    Given an instance of the $S_5$ word problem, we will construct an initial board state and a sequence of moves such that the final chessboard state encodes the output of that $S_5$ problem instance.

    Let $M$ denote the set of chess moves in the UCI, i.e., (source square, target square), notation.
    
    \noindent \textbf{Initial Board State.} We construct a chessboard similar to \Cref{fig:s5-chess-code} but with a black rook at $\m{a8}$ and black queens at $\m{b8}$ to $\m{e8}$.

    \noindent \textbf{Chess Move Sequence.}
    We then construct a finite function $f : S_5 \to M^*$ that encodes a permutation $\pi$ as a sequence of chess moves. We first factor each permutation $\pi$ to a sequence of transpositions $\tau_1(\pi) \cdots \tau_{m_\pi}(\pi)$.
    Each transposition $\tau \in T$ can in turn be expressed as a sequence of chess moves analogously to \Cref{fig:s5-chess-code}. For example, transposing items 1 and 3 can be expressed as the move sequence: (\m{a8}, \m{a7}), (\m{a1}, \m{b1}), (\m{c8}, \m{c6}), (\m{b1}, \m{a1}), (\m{a7}, \m{c7}), (\m{a1}, \m{b1}), (\m{c6}, \m{a6}), (\m{b1}, \m{a1}), (\m{c7}, \m{c8}), (\m{a1}, \m{b1}), (\m{a6}, \m{a8}), (\m{b1}, \m{a1}), which has the crucial property that it transposes \m{a8} with \m{c8}. We denote the mapping from transpositions to chess move sequences as $f : T \to M^*$. Putting it all together, we have
    \begin{equation*}
        f(\pi) = \bigcomp_{j=1}^{m_\pi} f(\tau_j(\pi)) .
    \end{equation*}
    To reduce a sequence of permutations $w \in S_5^*$, we let
    \begin{equation*}
        f(w) = \bigcomp_{i=1}^{n} f(w_i) .
    \end{equation*}

    \noindent \textbf{Putting It All Together.} We call our oracle for chess state tracking with the constructed initial board state and $f(w)$ as the sequence of chess moves. By construction, we can then return true if and only if the rook is at $\m{a8}$. The reduction can be implemented in $\NC^0$ because it is a simple elementwise mapping of the input tokens, and decoding from the output chessboard is a finite table lookup.
\end{proof}

As a fun aside, we note that the chess board constructed in the above proof is reachable in a standard chess game. The chess sequences encoding permutation sequences are all valid in the game of chess, except that they ignore the fact that repeated board states in chess technically lead to a draw.

Since $S_5$ is $\NC^1$-complete under $\AC^0$ reductions and $\NC^0 \subseteq \AC^0$, we have:

\begin{corollary}
    The chess state-tracking problem is $\NC^1$-complete under $\AC^0$ reductions.
\end{corollary}

Theorem 3.2 of \citet{feng2023towards} uses a similar reduction to prove formula evaluation is $\NC^1$-complete.
Reductions can be constructed for evaluating Python or tracking entities in a dialog, as suggested by \Cref{fig:s5-chess-code}.
As for chess, the task formatting for entity tracking affects its hardness. For instance, the formatting used by \citet{kim-schuster-2023-entity} in their Figure 1 is not $\NC^1$-complete, whereas the variant shown in our Figure 1 is. This underscores the value of theory for constructing examples of hard state tracking.

\section{SSMs Can be Simulated in $\TC^0$}

In this section, we show that the convolutional form of common variants of SSM can be simulated in $\TC^0$. Assuming the convolutional form of the model computes the same function as the recurrent form, this implies such SSMs cannot solve inherently sequential problems, despite their appearance of recurrence and statefulness. We first show containment in $\TC^0$ for non-gated SSMs (\Cref{thm:non-gated}), and then show the same holds for diagonal SSMs (\Cref{thm:diagonal}).

\subsection{Conditions for Linear SSMs in $\TC^0$}

Before characterizing specific SSM architectures, we first show that the complexity of computing transition matrix products essentially determines the complexity of simulating an SSM with a circuit family.

\begin{lemma} \label{lem:general}
    Let $M$ be a log-precision generalized linear SSM. Then there exists an $\L$-uniform $\TC^0$ circuit family that computes $M$'s convolutional form if:
    \begin{compactenum}
        \item For any integer interval $[j, k]$, the matrix product $\prod_{i = j}^k \barA_i$ can be computed in $\L$-uniform $\TC^0$ as a function of $\barA_j, \ldots, \barA_k$ (to $c \log n$ precision for any $c > 0$).
        \item For $1 \leq i \leq n$, $\barA_i, \barB_i, \mathbf C_i,$ and $\mathbf D_i$ can be computed in $\L$-uniform $\TC^0$ as a function of $\mathbf x_i$.
    \end{compactenum}
\end{lemma}

\begin{proof}
    Following the proof structure of \citet{merrill-sabharwal-2023-parallelism}, we describe how to construct a log-space bounded Turing machine $T_M$ that, given $\mathbf x_1, \ldots, \mathbf x_n$ as input, prints a circuit that simulates $M$ on this input. We first note that for all processing done before or after an SSM layer (projection, non-linearity, layer norm, etc.), $T_M$ can follow known simulations of such operations for transformers~\citep{merrill-sabharwal-2023-parallelism,merrill-sabharwal-2024-cot} to output a $\TC^0$ circuit simulating this processing. We thus focus on simulating an individual SSM layer.
    
    Recall from \cref{def:generalized-linear-ssm} that $M$'s convolutional form requires computing $\mathbf h_i = \sum_{j=1}^i \left( \prod_{k=j+1}^i \barA_k \right) \barB_j \mathbf x_j$ and $\mathbf y_i = \mathbf C_i \mathbf h_i + \mathbf D_i \mathbf x_i$. By the second precondition, $T_M$ can print a $\TC^0$ circuit that computes all matrices involved here. Further, by the first precondition, $T_M$ can also print a $\TC^0$ circuit that computes the innermost product in the computation of each hidden state $\mathbf h_i$, namely $\prod_{k=j+1}^i \barA_k$. It can now print a $\TC^0$ circuit to multiply the resulting product\footnote{Let $c \log n$ be the SSM's precision. We compute $\prod_k \barA_k$ to $c' \log n$ precision for a large enough $c'$ (similar to the proof of \cref{cor:iterated-float-product}) such that the full product $\left( \prod_k \barA_k \right) \barB_j \mathbf x_j$ is correct to at least $c \log n$ bits, as technically required by \Cref{def:flattening}.} with $\barB_j$ and $\mathbf x_j$, and then print a circuit to compute an iterated sum over the $i$ resulting vectors to compute $\mathbf h_i$ (cf.~iterated addition in \cref{app:arithmetic}). It can similarly print a (simpler) circuit to compute $\mathbf y_i$. Thus, the entire SSM layer can be simulated by an $\L$-uniform $\TC^0$ circuit.
\end{proof}

We will use \Cref{lem:general} to show that any non-gated or diagonal generalized linear SSM can be simulated in $\TC^0$.

\subsection{Non-Gated SSMs are in $\TC^0$} \label{sec:non-gated}

\begin{theorem}[Non-gated SSM] \label{thm:non-gated}
    Let $M$ be a log-precision generalized linear SSM such that, for any $i$,
    \begin{equation*}
        \barA_i = \barA, \quad
        \barB_i = \barB, \quad
        \mathbf C_i = \mathbf C, \quad
        \mathbf D_i = \mathbf D .
    \end{equation*}
    Then there exists an $\L$-uniform $\TC^0$ circuit family that computes $M$'s convolutional form.
\end{theorem}

\begin{proof}
    We prove this by showing that both conditions from \Cref{lem:general} are satisfied.
    Computing the matrix product reduces to powering $\barA^{k - j}$. Crucially, we can use the fact that matrix powering over floats is in $\L$-uniform $\TC^0$ (\Cref{cor:float-matrix-power}, extending \citealp{mereghetti2000threshold}).
    Finally, $\barA_i, \barB_i, \mathbf C_i,$ and $\mathbf D_i$ can be computed in $\L$-uniform $\TC^0$ because they are constants.
\end{proof}

As S4 satisfies the premises of \Cref{thm:non-gated}, we obtain:

\begin{corollary} \label{cor:s4}
    There exists an $\L$-uniform $\TC^0$ circuit family that computes S4's convolutional form.
\end{corollary}

\subsection{Diagonal SSMs are in $\TC^0$} \label{sec:diagonal}

\begin{theorem}[Diagonal SSM] \label{thm:diagonal}
Let $M$ be a log-precision generalized linear SSM where for $1 \leq i \leq n$:
\begin{compactenum}
    \item the transition matrix $\barA_i$ is diagonal, denoted $\diag(\bar{\mathbf{a}}_i)$ where $\bar{\mathbf a}_i \in \mathbb R^d$;
    \item each of $\bar{\mathbf a}_i, \barB_i, \mathbf C_i$ and $\mathbf D_i$ can be computed in $\L$-uniform $\TC^0$ as a function of $\mathbf x_i$.
\end{compactenum}
Then there exists an $\L$-uniform $\TC^0$ circuit family that computes $M$'s convolutional form.
\end{theorem}

\begin{proof}
    By the first condition, 
    $\prod_i \barA_i = \prod_i \diag(\bar{\mathbf{a}}_i)$.
    Iterated multiplication of diagonal matrices is reducible to several iterated scalar multiplications, placing this product in $\L$-uniform $\TC^0$ (\Cref{cor:iterated-float-product}). 
    The second condition from \Cref{lem:general} is satisfied by assumption. Thus, $M$'s convolutional form is computable in $\L$-uniform $\TC^0$.
\end{proof}

Since S6 satisfies the premises of \Cref{thm:diagonal}, we have:

\begin{corollary} \label{cor:s6-mamba}
    There exists an $\L$-uniform $\TC^0$ circuit family that computes S6's convolutional form (used by Mamba).
\end{corollary}

\begin{proof}
    For the first condition, note that S6's transition matrix $\barA_i$ is defined as $\exp(\delta_i \mathbf{A})$ for a fixed diagonal $\mathbf A$.
    The set of diagonal matrices is closed under scalar multiplication and matrix exponentiation, so $\barA_i$ is also diagonal.
    See \Cref{app:s6} for a proof that the second condition is satisfied by the S6 parameterization.
\end{proof}

\Cref{app:diagonalizable-s6} extends \cref{thm:diagonal} to hold even when $\{\barA_i\}$ are \textbf{simultaneously diagonalizable}, rather than just diagonal. Specifically, we prove the following generalization:

\begin{restatable}[Simultaneously diagonalizable SSM]{theorem}{diagonalizabletheorem} \label{thm:diagonalizable}
Let $\mathbf W$ be a fixed matrix.
Let $M$ be a log-precision generalized linear SSM such that, for $1 \leq i \leq n$,
\begin{compactenum}
    \item the transition matrix $\barA_i$ is computable to log precision by the expression $\mathbf{W}\diag(\bar{\mathbf{a}}_i)\mathbf{W}^{-1}$, where $\bar{\mathbf a}_i \in \mathbb R^d$;
    \item each of $\bar{\mathbf a}_i, \barB_i, \mathbf C_i$ and $\mathbf D_i$ can be computed in $\L$-uniform $\TC^0$ as a function of $\mathbf x_i$.
\end{compactenum}
Then there exists an $\L$-uniform $\TC^0$ circuit family that computes $M$'s convolutional form.
\end{restatable}

This, in turn, allows us to prove that a simultaneously diagonalizable transition matrix generalization of the S6 layer is also in $\L$-uniform $\TC^0$ (\cref{cor:diagonalizable-s6}).

\subsection{Discussion}

\Cref{thm:non-gated,thm:diagonal} establish that common SSM variants, like transformers, can only express solutions to problems in the class $\TC^0$. This means these SSMs cannot solve $\NC^1$-hard problems like evaluating boolean formulas or graph connectivity. In particular, it shows that they are limited as far as their state tracking capabilities as they are unable to compose permutations (solve the $S_5$ word problem):

\begin{corollary}
    Assuming $\TC^0 \neq \NC^1$, no log-precision SSM with the S4 or S6 architecture can solve the word problem for $S_5$ or any other $\NC^1$-hard problem.
\end{corollary}

In contrast, RNNs can easily express $S_5$ via standard constructions that encode finite-state transitions into an RNN \citep{minsky1954neural,merrill-2019-sequential}. This shows that SSMs cannot express some kinds of state tracking and recurrence that RNNs can. This tempers the claim from \citet[Lemma 3.2]{gu2021combining} that SSMs have the expressive power to simulate RNNs, which relied on the assumption that SSMs can have \emph{infinite depth}. In a more realistic setting with a bounded number of layers, our results show SSMs cannot express many state-tracking problems, including those which can be solved by fixed-depth RNNs.


\section{Extending the Expressive Power of SSMs}

We have shown that S4 and S6, despite their seemingly ``stateful'' design, cannot express problems outside $\TC^0$, which includes state tracking like $S_5$. We show how SSMs can be extended to close the gap in expressive power with RNNs, allowing them to express $S_5$.
Two simple extensions can bring about this increase in expressive power, assuming layer input dimension $k > 1$. First, adding a nonlinearity makes the SSM into an RNN, adding expressive power but degrading parallelism. On the other hand, allowing $\barA_i$ to be input-dependent makes the SSM more like a weighted finite automaton (WFA; \citealp{Mohri2009}), adding expressive power while remaining parallelizable.

\subsection{Via Nonlinearities}

One extension to the SSM is to add a nonlinearity, effectively making it an RNN.
We call this an RNN-SSM layer:
\begin{equation*}
    \mathbf h_i = \sgn \left( \barA \mathbf h_{i-1} + \barB x_i \right) .
\end{equation*}
A model with this architecture can solve the $S_5$ word problem when the input dimension $k > 1$:
\begin{theorem} \label{thm:rnn-ssm}
    For any regular language $L \subseteq \Sigma^*$ (including the word problem for $S_5$), there exists a one-layer log-precision RNN-SSM with $k = \abs{\Sigma}$ that recognizes $L$.
\end{theorem}
\begin{proof}
    The standard constructions for simulating automata with RNNs \citep[cf.][]{minsky1954neural,merrill-2019-sequential} apply. The condition $k = \abs{\Sigma}$ comes from needing to represent token types with linearly independent vectors.
\end{proof}


Adding a nonlinearity to the output of an SSM layer (as in Mamba) is not the same thing as an RNN-SSM. Rather, an RNN-SSM applies the nonlinearity at each recurrent update.
A downside of this approach is that it becomes nonlinear to parallelize the RNN-SSM computation graph with the SCAN algorithm used by linear SSMs \citep{BlellochTR90}. 

\subsection{Via Input-Dependent Transition Matrices} \label{sec:input-dependent}

Another way to get greater expressive power is to let the transition matrix $\barA_i$ be fully input-dependent, as explored by Liquid S4 \citep{hasani2023liquid}.
To illustrate this, we define a minimally different SSM called Input-Dependent S4 (IDS4) that achieves greater expressive power for state tracking.
Let $\pi_{\mathbf A} : \R^k \to \R^{d \times d}$ be some affine transformation where the output vector is interpreted as a $d \times d$ matrix,
and let 
$\barA_i = \pi_{\mathbf A}(\mathbf x_i)$.
Let $\barB, \mathbf C, \mathbf D$ be fixed (w.r.t.~$i$).
By \Cref{def:generalized-linear-ssm}, the IDS4 convolutional form computes an \emph{iterated product} of non-diagonal, input-dependent matrices:
\begin{equation*}
    \mathbf h_i = \sum_{j=1}^i \left( \prod_{k=j+1}^i \pi_{\mathbf A}(\mathbf x_i) \right) \barB \mathbf x_j .
\end{equation*}

In contrast to matrix powers or iterated products of diagonal matrices, iterated products of \emph{general} matrices cannot be computed in $\TC^0$ \citep{mereghetti2000threshold}.
This means that the arguments from \Cref{thm:non-gated,thm:diagonal} will not go through for IDS4. In fact, we can show IDS4 gains expressive power beyond $\TC^0$:

\begin{figure*}
    \centering
    \includegraphics[width=0.95\textwidth]{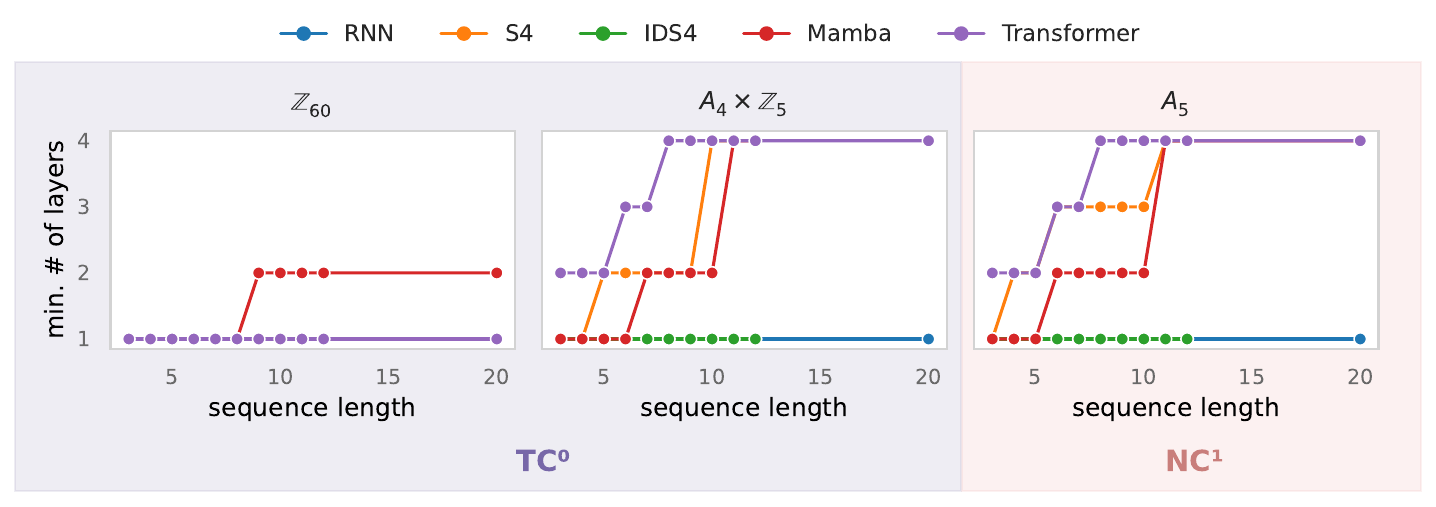}
    \caption{Minimum number of layers (lower is better) required to attain $>90\%$ validation accuracy on group multiplication problems by sequence length and group. RNN and IDS4 models of constant depth can solve arbitrarily long sequences, while transformer, S4, and Mamba models require depths monotonically increasing in sequence length.}
    \label{fig:group-data}
\end{figure*}

\begin{theorem} \label{thm:wfa-ssm}
    For any regular language $L \subseteq \Sigma^*$ (including the word problem for $S_5$), there exists a one-layer log-precision IDS4 SSM with $k = \abs{\Sigma}$ that recognizes $\$L$, where $\$ \not\in \Sigma$ is a special beginning-of-string symbol.
\end{theorem}

\begin{proof}
    It suffices to show that IDS4 can simulate a deterministic finite automaton (DFA). We do this via a transition monoid construction. For any $w \in \Sigma^*$, let $\delta_w : Q \to Q$ be the function mapping a state to its eventual destination state after $w$ is read from that state. For any DFA, this set of functions forms a finite monoid (the \emph{transition monoid}) under composition,
    following from the Myhill-Nerode theorem \citep{hopcroft2001introduction}.
    Further, each monoid element $\delta_w$ can be represented as a boolean transition matrix, making matrix multiplication isomorphic to monoid composition.
    Computing the transition monoid of a DFA allows recognizing valid words: compute the monoid element for a word by multiplying the elements for its tokens and then check whether the initial state maps to an accepting state.

    Fix a DFA and its transition monoid $\delta$.
    To complete the proof, we show there exists an SSM that, for all $w \in \Sigma^*$, computes $\delta_w$ given input $x=\$w$.
    Let $\barA_i$ be the transition matrix representation of $\delta_{x_i}$. Matrix multiplication is isomorphic to composition of transition monoid elements.
    We view indices in $\mathbf h_i$ as states and define $\barB \$$ as $1$ at the initial state $q_0$ and $0$ elsewhere. For other $\sigma$, let $\barB \sigma = \vec 0$.
    This yields the following convolutional form:
    \begin{equation*}
        \mathbf h_i
        = \left( \prod_{k=2}^{i} \barA_i \right) \mathbf B \$
        \equiv \left( \bigcomp_{k=2}^{i} \delta_{x_k} \right) (q_0) .
    \end{equation*}
    Since $x = \$w$, we conclude that $\mathbf h_{\abs{x}} \equiv \delta_w(q_0)$.
\end{proof}



\subsection{Discussion}

\Cref{thm:rnn-ssm,thm:wfa-ssm} show that two minimal extensions of the SSM enable expressive power outside $\TC^0$, allowing the model to solve hard state-tracking problems:

\begin{corollary}
    There exist a one-layer log-precision RNN-SSM and WFA-SSM that express the word problem for $S_5$ (with a beginning-of-string symbol), and these these SSMs cannot be simulated in $\TC^0$.
\end{corollary}

But would these variants of SSMs be feasible to use in practice? Besides expressive power, there are two competing practical concerns that might make these extensions problematic: parallelism and the impact on learning dynamics.

\noindent \textbf{Parallelism.}
To be used effectively in an LLM, a model architecture must be parallelizable on practical hardware. Architectures in $\TC^0$ are parallelizable by design \citep{merrill-sabharwal-2023-parallelism}, but architectures in $\NC^1$ may still be parallelizable to log depth even if they cannot be parallelized to constant depth. For IDS4, the bottleneck would be computing iterated matrix product with a log-depth computation graph. This could be achieved with the SCAN algorithm \citep{BlellochTR90} similar to S4 and S6. In contrast, it is less clear how to parallelize a model with a nonlinearity.

\textbf{Learning Dynamics.} Another potential concern for IDS4 is that learning dynamics could be degraded. In particular, an iterated product of matrices may lead to vanishing or exploding gradients. However, this is already potentially an issue for the S6 architecture, where the selective gating involves computing an iterated product of scalars.

\section{Can SSMs Learn Permutations in Practice?}

Having established theoretical limitations of SSMs for state tracking, we empirically test
how well SSMs can learn such tasks,
focusing on the $A_5$ word problem. Since this problem is $\NC^1$-complete and transformers, S4, and Mamba can only express functions in $\TC^0$, these models should require a depth that grows with the input length to solve this problem.


\noindent \textbf{Task.}
We model word problems (see \Cref{sec:word-problems}) as a token-tagging task. Models are given as input a sequence $g_0g_1g_2 \cdots g_n$ drawn from one of $A_5$, $A_4 \times \mathbb{Z}_5$, or $\mathbb{Z}_{60}$. At each step $i$, the label is the product of the first $i$ elements of the sequence. Modeling the problem as a tagging task rather than as a sequence classification task provides the models with more supervision during training, making it as easy as possible to learn the correct function. We tokenize inputs such that each element gets a unique token.

\noindent \textbf{Models.}
We train a transformer as a $\TC^0$ baseline, an RNN that we expect can perform state tracking, and three SSMs: S4 \citep{gu2022efficiently}, Mamba \citep{gu2023mamba}, and IDS4 (\Cref{sec:input-dependent}).
For IDS4, we initialize the affine projection $\alpha$ as a random normal centered around the identity: $\alpha(\mathbf x_i) \sim \mathbf I + \mathcal N(0, \sigma^2)$. 
This ensures that, at initialization, input-dependent transitions tend to propagate the previous state, which we expect to aid learning efficiency.

\noindent \textbf{Experimental Setup.}
We train models on sequences of length $n$ for successively larger values of $n$ and report full-sequence accuracy on a test set.\footnote{We always include all $3600$ pairwise sequences of length $2$ in the training data along with the training split of length-$n$ sequences.} To validate the prediction that SSMs and transformers require growing depth to solve longer $A_5$ word problems, we plot the minimum depth with 90\% test accuracy as a function of input sequence length. 

\noindent \textbf{Results.}
\Cref{fig:group-data} shows single-layer RNN and IDS4 models learn the word problem for arbitrarily long sequences for all three groups. In contrast, transformer, S4, and Mamba models require depth monotonically increasing in sequence length to attain good test accuracy for the non-commutative groups. We draw three conclusions from this:

    1. As expected, S4 and Mamba show the same limitations as transformers on the $A_5$ word problem. Longer $A_5$ sequences require deeper models, consistent with these models being in $\TC^0$. In contrast, RNNs (\Cref{thm:rnn-ssm}) and IDS4 (\Cref{thm:wfa-ssm}) can efficiently solve
    the $A_5$ word problem.
    
    2. Transformers, S4, and Mamba require greater depth even for $A_4 \times \mathbb{Z}_5$, which can be theoretically expressed by $\TC^0$ circuits. Although transformer and Mamba models of a given depth perform as good or better on $A_4 \times \mathbb{Z}_5$ as they on $A_5$, they still require increasingly many layers to handle proportionally longer sequences. There are two possible interpretations of this. First, it could be that while these word problems are expressible in $\TC^0$, they cannot be expressed by S4, Mamba, or transformers (which can each likely recognize only a proper subset of $\TC^0$). On the other hand, it is possible that these word problems \emph{are} expressible by transformers, S4, and Mamba but that effectively learning a constant-depth solution is difficult.
    
    3. Despite this limitation, S4 and Mamba appear \emph{empirically better} than transformer at approximate state tracking on the non-commutative tasks. For length-$n$ sequences from $A_4\times\mathbb{Z}_5$ or $A_5$, the transformer requires at least as many (and frequently more) layers as S4 or Mamba.

\section{Conclusion}

We formally analyzed a family of generalized linear SSMs and showed that, like transformers, common SSM variants including S4 and Mamba can only express computation within the complexity class $\L$-uniform $\TC^0$ of highly parallel computations. This means they cannot solve inherently sequential problems like graph connectivity, boolean formula evaluation, and---of particular interest for state tracking---the permutation composition problem $S_5$. $S_5$ can be naturally expressed by true recurrent models like RNNs and captures the essence of hard state tracking due to its $\NC^1$-completeness. In practice, one-layer RNNs can easily learn a task capturing $S_5$ while linear SSMs require depth growing with the sequence length. These results reveal that S4, Mamba, and related SSMs cannot truly track state: rather, they can only solve simple state-tracking problems for which shallow shortcuts exist \citep{liu2023transformers}.

On the other hand, we showed that an input-dependent SSM similar to \citepos{hasani2023liquid} Liquid S4 can both express and learn the $S_5$ word problem, providing evidence that the expressiveness limitations of current SSMs can be overcome.
Ultimately, this line of work could unlock new neural architectures that balance the parallelism of transformers and SSMs with full expressive power for state tracking, enabling LLMs that can benefit from scale while enjoying a greater capacity to reason about games, code, and language.

\section*{Impact Statement}

This paper aims to advance the foundational understanding of state-space architectures for deep learning.
Such work can affect the development and deployment of deep learning models in a variety of ways, which in turn can have societal impacts. However, we find it difficult to meaningfully speculate about or anticipate these downstream impacts here.

\section*{Acknowledgments}
This work benefited from discussions with and valuable feedback from Chris Barker, Stefano Ermon, and Charles Foster.
It was supported in part through the NYU IT High Performance Computing resources, services, and staff expertise. It
was funded by NSF award 1922658, and WM was
supported by an NSF graduate research fellowship, AI2, and Two Sigma.

\bibliographystyle{icml2024}
\bibliography{references}

\appendix

\section{Floating-Point Arithmetic}
\label{app:arithmetic}

Our results use the \textbf{log-precision floating point} model used by \citet{merrill2023logic} to analyze transformers. For some fixed constant $c \in \Z^+$, a $c \log n$ precision float is a tuple $\langle m, e \rangle$ where $m, e$ are signed integers together taking $c \log n$ bits. Using $|x|$ to mean the number of bits used to represent integer $x$, this float represents the value $m \cdot 2^{e - |m| + 1}$.

Unlike for integers, arithmetic operations over log-precision floats are not closed. That is, the product $\phi_1 \times \phi_2$ of two $p$-precision floats is a well-defined number but may not be exactly representable as a $p$-precision float.
It is thus necessary to define approximate versions of these operations when formalizing log-precision floating-point arithmetic.
To this end, \citet{merrill-sabharwal-2023-parallelism} define a natural notion of approximate iterated addition over log-precision floats and show that it is computable in $\L$-uniform $\TC^0$.
We can naturally apply their definition of iterated addition for floats to matrices of floating points, defining iterated summation over matrices of datatype $\D$ as the result of treating the numbers as reals, performing exact arithmetic, and casting the exact output $\phi$ back to $\D$, denoted $\cast_\D(\phi)$. Formally:

\begin{definition}[Iterated $\D$-matrix sum; \citealp{merrill-sabharwal-2023-parallelism}]
    For matrices $\mathbf M_1, \ldots, \mathbf M_n$ over $\D$ with the same size, their \emph{iterated $\D$-sum} is
    \begin{equation*}
        \bigoplus_{i=1}^z \mathbf M_i \, \triangleq \, \cast_\D \left( \sum_{i=1}^z \cast_\R( \mathbf M_i ) \right) .
    \end{equation*}
\end{definition}
Here $\cast_\R$ converts a number in $\D$ to the corresponding real number. $\D$ is implicit in the notations $\cast_\R$ and $\bigoplus$.
Integer addition can be obtained as a special case for $1$-dimensional matrices.
We can also analogously defined iterated summation, which will be necessary for formalizing SSMs:

\begin{definition}[Iterated $\D$-matrix product]
    \label{defn:iterated-D-product}
    For square matrices $\mathbf M_1, \ldots, \mathbf M_z$ over $\D$, their \emph{iterated $\D$-product} is
    \begin{equation*}
        \bigotimes_{i=1}^z \mathbf M_i \, \triangleq \, \cast_\D \left( \prod_{i=1}^z \cast_\R( \mathbf M_i ) \right) .
    \end{equation*}
\end{definition}


\citet{merrill-sabharwal-2023-parallelism} showed that iterated addition from for log-precision floats is in $\L$-uniform $\TC^0$. It naturally follows from their argument that \textbf{iterated addition} over log-precision float matrices is also in $\L$-uniform $\TC^0$.
In general, iterated matrix products are not necessarily computable in $\TC^0$.
However, we extend the arguments of \citet{hesse2001division} and \citet{mereghetti2000threshold} for integers to show that two special cases (iterated scalar multiplication and matrix powering) over log-precision floats are also computable in $\L$-uniform $\TC^0$.

Finally, we define a canonical value for a compositional arithmetic expression over floats that enjoys the associative property.

\begin{definition}[Flattened expression evaluation] \label{def:flattening}
Let $\phi$ be a compositional expression over floats, which may contain alternating sums and products as well as other operations like $\exp$. We define the \emph{canonical value} of $\phi$ as the value returned by the computation graph obtained by flattening all adjacent sums into a single sum (and analogously for products).
\end{definition}

\Cref{def:flattening} has the nice effect of making \Cref{defn:iterated-D-product} associative. The only results that rely on this assumption are our analysis of diagonalizable SSMs in \Cref{app:diagonalizable-s6}.
We also deal with the details of this assumption in \Cref{lem:general}, though the proof there also goes through directly without handling these details.

\subsection{Complexity of Iterated Scalar Multiplication}

The first special case of iterated matrix products we analyze is when the matrices are simply scalars (or, w.l.o.g., diagonal matrices). In this case, the iterated product can be computed in $\L$-uniform $\TC^0$.

\begin{lemma}[Iterated $\D$-product]
    \label{lem:iterated-D-product}
    Let $\phi_1, \ldots, \phi_z \in \D$ be such that $z \leq n$ and each $\phi_i$ can be represented as an $n$-bit integer. If operators $\cast_\D$ and $\cast_\R$ are in $\L$-uniform $\TC^0$, then the iterated $\D$-product $\bigotimes_{i=1}^z \phi_i$ can be computed in $\L$-uniform $\TC^0$.
\end{lemma}

\begin{proof}
    By preconditions of the lemma, we can compute $y_i = \cast_\R( \phi_i )$ for each $i$ in $\L$-uniform $\TC^0$. Since each $\phi_i$ is equivalent to an $n$-bit integer, $y_i$ can be viewed as an $n$-bit integer. The iterated integer product $y = \prod_{i=1}^z y_i$ can be computed with an $\L$-uniform $\TC^0$ circuit \cite{hesse2001division}. Finally, by a precondition of the lemma, we can cast the result back to $\D$, i.e., compute $\cast_\D ( y )$ which equals the iterated $\D$-product $\bigotimes_{i=1}^z \phi_i$, with an $\L$-uniform $\TC^0$ circuit.
\end{proof}

\begin{lemma}[Iterated float product]
    \label{cor:iterated-float-product}
    Let $\phi_1, \ldots, \phi_z$ be $c \log n$ precision floats and $z \leq n$. Then the iterated float product $\bigotimes_{i=1}^z \phi_i$ can be computed in $\L$-uniform $\TC^0$.
\end{lemma}

\begin{proof}
    The idea is to convert (by scaling up) the sequence of $\phi_i$ to another sequence of floats that are all representable as integers, apply \cref{lem:iterated-D-product}, reverse the scaling, and cast the result back to a $c \log n$ precision float.

    Let $e$ be the smallest exponent across all $\phi_i$ and $q = \max \{0, -e\}$. Construct re-scaled floats $\psi_i = \phi_i 2^q$ by adding $q$ to the exponent of $\phi_i$, using up to $c \log n$ additional bits in the exponent if necessary to keep the computation exact. Note that $e$, $q$, and all $\psi_i$ can easily be computed exactly by an $\L$-uniform $\TC^0$ circuit as they involve fixed-arity arithmetic operations. Further, by construction, every $\psi_i$ has a non-negative exponent and thus represents an integer.
    
    The maximum number representable by each $c \log n$ precision float $\phi_i$ is upper bounded by $2^{n^c}$. Thus, the maximum number representable by each entry $\psi_i$ is $2^{n^c} \times 2^q = 2^{n^c + q}$. Let $m = n^c + q$. It follows that each $\psi_i$ can be equivalently represented as an $m$-bit integer. Further, this integer can be computed by left-shifting the mantissa of $\psi_i$ by a number of bits equal to the value of the exponent of $\psi_i$ (which is non-negative). Finally, this left-shift, and thus the $\cast_\R$ operation over $m$-precision floats, can be easily computed by an $\L$-uniform threshold circuit of size $\poly(m)$. In the other direction, casting from reals to $m$-precision floats can also be easily accomplished by an $\L$-uniform threshold circuit of size $\poly(m)$.

    Observing that $\psi_1, \ldots, \psi_z$ is a sequence of
    floats each representable as an $m$-bit integer,
    we now apply \cref{lem:iterated-D-product} with $\D$ being `float' to conclude that iterated float product $\tau = \bigotimes_{i=1}^z \psi_i$ can be computed by an $\L$-uniform threshold circuit of size $\poly(m)$. Since
    $m \leq 2n^c$,
    this circuit is also of size $\poly(n)$.

    Finally, to compute the original iterated float product $\bigotimes_{i=1}^z \phi_i$, we divide $\tau$ by $2^{qz}$. This can be accomplished by subtracting $qz$ from the exponent of $\tau$; again, we do this computation exactly, using up to $(c+1) \log n$ additional bits in the exponent if necessary. We then cast the resulting float back to a $c \log n$ precision float. All this can be done in $\L$-uniform $\TC^0$, finishing the proof that $\bigotimes_{i=1}^z \phi_i$ can be computed in $\L$-uniform $\TC^0$.
\end{proof}

\subsection{Complexity of Matrix Powering}

The second special case we analyze is matrix powering: i.e., a matrix product where all the matrices being powered are the same.
\citet{mereghetti2000threshold} showed that when the datatype $\D$ is $n$-bit integers, one can compute $\mathbf M^n$ in $\TC^0$. We note that their construction also works for computing $\mathbf M^z$ for any $z \leq n, z \in \Z^+$. Further, as they remark, their construction can, in fact, be done in \emph{uniform} $\TC^0$. Specifically, we observe most of their construction involves sums and products of constantly many $n$-bit integers, which can be done in $\L$-uniform $\TC^0$. The only involved step is dividing a polynomial of degree (up to) $n$ by a polynomial of degree (up to) $d-1$ and returning the remainder. It turns out that this ``polynomial division with remainder'' operation can also be performed in $\L$-uniform $\TC^0$ (see Corollary 6.5 of \citealp{Hesse2002UniformCT} and an explanation in \cref{subsec:polynomial-division}). We thus have the following extension of \citeauthor{mereghetti2000threshold}'s result:

\begin{lemma}[Integer matrix power, derived from \citealp{mereghetti2000threshold}]
    \label{lem:integer-matrix-power}
    Let $d \in \Z^+$ be a fixed constant.
    Let $\mathbf M$ be a $d \times d$ matrix over $n$-bit integers and $z \leq n, z \in \Z^+$. Then integer matrix power $\mathbf M^z$ can be computed in $\L$-uniform $\TC^0$.
\end{lemma}

We extend this to matrix powers over $\D$ rather than integers:

\begin{lemma}[$\D$-matrix power]
    \label{lem:D-matrix-power}
    Let $d \in \Z^+$ be a fixed constant.
    Let $\mathbf M$ be a $d \times d$ matrix over a datatype $\D$ with entries equivalently representable as $n$-bit integers. Let $z \leq n, z \in \Z^+$. If operators $\cast_\D$ and $\cast_\R$ are in $\L$-uniform $\TC^0$, then $\D$-matrix power $\mathbf M^z$ can be computed in $\L$-uniform $\TC^0$.
\end{lemma}

\begin{proof}
    By preconditions of the lemma, we can compute $\cast_\R( \mathbf M )$ in $\L$-uniform $\TC^0$. Since the entries of $\mathbf M$ are equivalent to $n$-bit integers, $\cast_R( \mathbf M )$ can be viewed as a $d \times d$ integer matrix of $n$-bit integers. By \cref{lem:integer-matrix-power}, we can compute $\cast_R( \mathbf M )^z$ using an $\L$-uniform $\TC^0$ circuit. Finally, by a precondition of the lemma, we can cast the result back to $\D$, i.e., compute $\cast_\D(\cast_\R( \mathbf M )^z)$ which equals $\mathbf M^z$, with an $\L$-uniform $\TC^0$ circuit.
\end{proof}

\begin{lemma}[Float matrix power]
    \label{cor:float-matrix-power}
    Let $d, c \in \Z^+$ be fixed constants.
    Let $\mathbf M$ be a $d \times d$ matrix over $c \log n$ precision floats. Let $z \leq n, z \in \Z^+$. Then float matrix power $\mathbf M^z$ can be computed in $\L$-uniform $\TC^0$.
\end{lemma}

\begin{proof}
    The idea is to convert (by scaling up) $\mathbf M$ to another float matrix all whose entries are representable as integers, apply \cref{lem:D-matrix-power}, reverse the scaling, and cast the result back to $c \log n$ precision floats.

    Let $e$ be the smallest exponent across all float entries of $\mathbf M$ and $q = \max \{0, -e\}$. Construct a re-scaled float matrix $\tilde{\mathbf M} = \mathbf M 2^q$ by adding $q$ to the exponent of every entry of $\mathbf M$, using up to $c \log n$ additional bits in the exponent if necessary to keep the computation exact. Note that $e$, $q$, and $\tilde{\mathbf M}$ can easily be computed exactly by an $\L$-uniform $\TC^0$ circuit as they involve fixed-arity arithmetic operations. Further, by construction, $\tilde{\mathbf M}$ has non-negative exponents in all its float entries. Thus, every entry of $\tilde{\mathbf M}$ represents an integer.
    
    The maximum number representable by each $c \log n$ precision float in $\mathbf M$ is upper bounded by $2^{n^c}$. Thus, the maximum number representable by each entry of $\tilde{\mathbf M}$ is $2^{n^c} \times 2^q = 2^{n^c + q}$. Let $m = n^c + q$. It follows that each entry $\phi$ of $\tilde{\mathbf M}$ can be equivalently represented as an $m$-bit integer. Further, this integer can be computed by left-shifting the mantissa of $\phi$ by a number of bits equal to the value of the exponent of $\phi$ (which is non-negative). Finally, this left-shift, and thus the $\cast_\R$ operation over $m$-precision floats, can be easily computed by an $\L$-uniform threshold circuit of size $\poly(m)$. In the other direction, casting from reals to $m$-precision floats can also be easily accomplished by an $\L$-uniform threshold circuit of size $\poly(m)$.

    Note that $2^q \in [0, n^c]$ and hence $m \in [n^c, 2n^c]$. In particular, $m \geq n$. Thus $z \leq n$ (a precision) implies $z \leq m$. Observing that $\tilde{\mathbf M}$ is a matrix of
    floats each representable as an $m$-bit integer,
    we now apply \cref{lem:D-matrix-power} with $\D$ being `float' to conclude that float matrix power $\tilde{\mathbf M}^z$ can be computed by an $\L$-uniform threshold circuit of size $\poly(m)$. Since
    $m \leq 2n^c$,
    this circuit is also of size $\poly(n)$.

    Finally, to compute $\mathbf M^z$, we first divide each entry of $\tilde{\mathbf M}^z$ by $2^{qz}$. This can be accomplished by subtracting $qz$ from the exponent of each entry of $\tilde{\mathbf M}$; again, we do this computation exactly, using up to $(c+1) \log n$ additional bits in the exponent if necessary. We then cast all entries of the resulting matrix back to $c \log n$ precision floats. All this can be done in $\L$-uniform $\TC^0$, finishing the proof that $\mathbf M^z$ can be computed in $\L$-uniform $\TC^0$.
\end{proof}

\subsection{$\L$-Uniformity of Polynomial Division in $\TC^0$}
\label{subsec:polynomial-division}

\citet{Hesse2002UniformCT} state that polynomial division is in $\L$-uniform $\TC^0$ in Corollary 6.5. For historical reasons, this claim is preceded by weaker claims in older papers. We briefly clarify this situation to help understand why the stronger claim is valid.

\citet{reif1992threshold} establish that polynomial division can be performed in $\P$-uniform $\TC^0$, whereas we state our results for $\L$-uniform $\TC^0$, which is a smaller class. However, the only issue preventing the polynomial division result from originally going through in the $\L$-uniform case is that, at the time of \citeauthor{reif1992threshold}'s publication, it was not known whether integer division and iterated integer multiplication are computable in $\L$-uniform $\TC^0$. However, \citet{hesse2001division} later proved exactly this. Combining the two results, Theorem 3.2 of \citet{reif1992threshold} goes through even with $\L$-uniformity (not just $\P$-uniformity). Its Corollary 3.3 then allows us to conclude that integer polynomial division can be solved by $\L$-uniform $\TC^0$ circuits because the output of integer polynomial division is an analytic function whose Taylor expansion has a finite number of terms \citep{reif1992threshold}.



\section{S6 Parameterization} \label{app:s6}

To justify that the S6 architecture used by Mamba is computable in $\TC^0$, we justify that $\barA_i, \barB_i, \mathbf C_i, \mathbf D_i$ can be computed as a function of $\mathbf x_i$ in $\TC^0$.

We begin by summarizing how exactly is S6 parameterized. S6 first defines continuous-time parameters:
\begin{compactenum}
    \item $\mathbf A$ is a fixed, diagonal matrix that is invertible (each $a_{ii} \neq 0$);
    \item $\mathbf B_i = \pi_{\mathbf B}(\mathbf x_i)$ is computed via a projection;
    \item $\mathbf C_i = \pi_{\mathbf C}(\mathbf x_i)$ is computed via a projection;
    \item $\mathbf D_i = \mathbf I$ .
\end{compactenum}
Next, we need to discretize the matrices $\mathbf A$ and $\mathbf B$.
S6 does this using an input-dependent discretization factor $\delta_i$:
\begin{equation*}
    \delta_i = \mathrm{softplus}(\delta + \pi_\delta(\mathbf x_i)) .
\end{equation*}
The discretized matrices are then defined as:
\begin{align*}
    \barA_i &= \exp(\delta_i \mathbf A) \\
    \barB_i &= (\delta_i \mathbf A)^{-1} \left( \barA_i - \mathbf I \right) \delta_i \mathbf B_i .
\end{align*}

It is clear to see that the diagonalizability condition of \Cref{thm:diagonal} is satisfied because $\barA_i$ itself is diagonal. Additionally, all the relevant matrices can be computed in $\TC^0$.

\begin{proposition}
    $\barA_i, \barB_i, \mathbf C_i,$ and $\mathbf D_i$ can all be computed as functions of $\mathbf x_i$ in $\L$-uniform $\TC^0$.
\end{proposition}

To prove this, observe that
$\mathbf A, \mathbf B_i, \mathbf C_i, \mathbf D_i$ can all be computed in $\L$-uniform $\TC^0$ because they are either constants or linear transformations of $\mathbf x_i$. To justify that $\barA_i$ and $\barB_i$ can be computed in $\L$-uniform $\TC^0$, we just need to justify that we can invert diagonal matrices and compute $\mathrm{softplus}$ and $\exp$ in $\L$-uniform $\TC^0$.

\begin{lemma} \label{lem:invert_diagonal}
    Diagonal matrices over log-precision floats can be inverted in $\L$-uniform $\TC^0$.
\end{lemma}

\begin{proof}
    Inverting a diagonal matrix just involves forming the reciprocal along the diagonal. Scalar reciprocals can be approximated to error at most $2^{-n^c}$ (for any $c$) in $\TC^0$ \citep{Hesse2002UniformCT}. This means we can compute the reciprocal of a log-precision float (cf. \Cref{app:arithmetic}) exactly up to log precision.
\end{proof}

In \Cref{sec:nonlinearities}, we show that we can compute the nonlinearities $\exp$ and $\softplus$ over a bounded domain in $\TC^0$.

\section{Diagonalizable SSMs} \label{app:diagonalizable-s6}

We extend \cref{thm:diagonal} to cover the case when the SSMs transition matrices are \emph{simultaneously diagonalizable}, rather than just diagonal. This requires us to note that when working with log-precision floating point representations of matrices, a diagonal matrix $\mathbf{A}$ and its diagonalized decomposition $\mathbf{W}\diag(\mathbf{a})\mathbf{W}^{-1}$ are numerically substitutable.

\diagonalizabletheorem*

\begin{proof}
    When the first condition is satisfied, the following equality holds over log-precision floats:
    \begin{equation*}
        \prod_i \barA_i = \prod_i \left( \mathbf{W}\diag(\bar{\mathbf{a}}_i)\mathbf{W}^{-1} \right) .
    \end{equation*}
    By the associativity of $\mathbb D$-matrix products, we can remove the parentheses to get
    \begin{align*}
        \prod_i \barA_i
        &= \prod_i\mathbf{W}\diag(\bar{\mathbf{a}}_i)\mathbf{W}^{-1} \\
        &= \mathbf{W} \left[\prod_i \diag(\bar{\mathbf{a}}_i)\right]\mathbf{W}^{-1}.
    \end{align*}
    Iterated multiplication of diagonal matrices is reducible to several iterated scalar multiplications, which is in $\L$-uniform $\TC^0$ (\Cref{cor:iterated-float-product}). 
    Then the product of all $\barA_i$ is the product of three $\L$-uniform $\TC^0$-computable matrices, so is itself $\L$-uniform $\TC^0$-computable.
    The second condition from \Cref{lem:general} is satisfied by assumption. Thus, the convolutional form for $M$ can be computed in $\L$-uniform $\TC^0$.
\end{proof}

\subsection{Diagonalizable S6}

We can define an extension of S6 which satisfies these conditions to show that it is also in $\L$-uniform $\TC^0$.

\begin{definition}
    Diagonalizable S6 has continuous-time parameters:
    \begin{compactenum}
        \item $\mathbf{A}$ is a fixed matrix diagonalizable as $\mathbf{W}\diag(\mathbf{a})\mathbf{W}^{-1}$ that is invertible (each $a_{ii} \neq 0$);
        \item $\mathbf{B}_i = \pi_\mathbf{B}(\mathbf{x}_i)$ is computed via a projection;
        \item $\mathbf{C}_i = \pi_\mathbf{C}(\mathbf{x}_i)$ is computed via a projection;
        \item $\mathbf{D} = \mathbf{I}$.
    \end{compactenum}
    As in the standard S6, the discretization of $\mathbf{A}$ and $\mathbf{B}$ is done by an input-dependent discretization factor $\delta_i$:
    $$
        \delta_i = \mathrm{softplus}(\delta + \pi_\delta(\mathbf{x}_i)).
    $$
    The discretized matrices are then defined as
    \begin{align*}
        \barA_i &= \exp(\delta_i \mathbf{A}), \\
        \barB_i &= (\delta_i\mathbf{A})^{-1}(\barA_i - \mathbf{I})\delta_i\mathbf{B}_i.
    \end{align*}
\end{definition}

To prove that $\barA_i$ and $\barB_i$ have the necessary properties, we first introduce some lemmas dealing with matrix-valued functions of diagonalizable matrices.

\begin{lemma}
    If a matrix $\mathbf{A}$ is diagonalizable, then we can substitute its diagonalized decomposition $\mathbf{W}\diag(\mathbf{a})\mathbf{W}^{-1}$ in a computation graph over log-precision floats involving $\mathbf{A}$ without incurring any meaningful error.
\end{lemma}
\begin{proof}
    Let $\mathbf{A}$ be diagonalizable. Then there exists invertible $\mathbf{W}$ and diagonal $\diag(\mathbf{a})$ such that $\mathbf{A} = \mathbf{W}\diag(\mathbf{a})\mathbf{W}^{-1}$. Note that the product of a fixed number of matrices is in $\L$-uniform $\TC^0$, and so the first $c \log n$ bits of $\mathbf{A}$ and $\mathbf{W}\diag(\mathbf{a})\mathbf{W}^{-1}$ are identical.
\end{proof}

\begin{lemma} \label{lem:factor_scalar}
    Let $\mathbf{A}$ be diagonalizable as $\mathbf{W}\diag(\mathbf{a})\mathbf{W}^{-1}$, where $\mathbf{a} \in \mathbb{R}^d$.
    Then $c\cdot\mathbf{A}$ is simultaneously diagonalizable with $\mathbf{A}$ via $c \cdot \mathbf{A} = \mathbf{W}c\cdot\diag(\mathbf{a})\mathbf{W}^{-1}$.
\end{lemma}
\begin{proof}
    Scalar multiplication commutes around matrix multiplication.
\end{proof}

\begin{lemma} \label{lem:factor_exp}
    Let $\mathbf{A}$ be diagonalizable as $\mathbf{W}\diag(\mathbf{a})\mathbf{W}^{-1}$, where $\mathbf{a} \in \mathbb{R}^d$. Then $\exp(\mathbf{A}) = \mathbf{W} \exp(\diag(\mathbf{a})) \mathbf{W}^{-1}$.
\end{lemma}

\begin{proof}
    The matrix exponential is defined as a power series, so for diagonalizable $\mathbf{A}$ it follows that
    \begin{align*}
        \exp(\mathbf{A}) &= \exp(\mathbf{W}\diag(\mathbf{a})\mathbf{W}^{-1}) \\
        &= \sum_{k=0}^\infty \frac{1}{k!} (\mathbf{W}\diag(\mathbf{a})\mathbf{W}^{-1})^k \\
        &= \sum_{k=0}^\infty \frac{1}{k!} \mathbf{W}\diag(\mathbf{a})^k\mathbf{W}^{-1} \\
        &= \mathbf{W}\left(\sum_{k=0}^\infty \frac{1}{k!}\diag(\mathbf{a})^k\right)\mathbf{W}^{-1} \\
        &= \mathbf{W}\exp(\diag(\mathbf{a}))\mathbf{W}^{-1}. \tag*{\qedhere}
    \end{align*}
\end{proof}

The expressions in \Cref{lem:factor_exp} are equivalent not just over real numbers but also over log-precision floats. This is because we know both expressions can be approximated in $\TC^0$ with error at most $2^{-n^c}$, which means the $c \log n$ bits of the approximation must be equivalent.

\begin{lemma} \label{lem:invert_diagonalizable}
Diagonalizable matrices over log-precision floats can be inverted in $\L$-uniform $\TC^0$.
\end{lemma}

\begin{proof}
    Let $\mathbf{A} = \mathbf{W}\diag(\mathbf{a})\mathbf{W^{-1}}$. Then
    $\mathbf{A}^{-1} = \mathbf{W}^{-1}\diag(\mathbf{a})^{-1}\mathbf{W}$. We are guaranteed that each of these matrices exists, and furthermore by \Cref{lem:invert_diagonal} we know that $\diag(\mathbf{a})^{-1}$ is computable in $\L$-uniform $\TC^0$. Their product, involving a finite number of additions and multiplies, is also computable in $\L$-uniform $\TC^0$.
\end{proof}

\begin{proposition}
    $\barA_i$ and $\barB_i$ can be computed as functions of $\mathbf{x}_i$ in $\L$-uniform $\TC^0$.
\end{proposition}

\begin{proof}
    We first show that $\barA_i$ is $\L$-uniform $\TC^0$ computable. By definition,
    $$
        \barA_i = \exp(\delta_i \mathbf{A}).
    $$
    By \Cref{cor:exp-log-softplus}, $\delta_i$ is computable in $\L$-uniform $\TC^0$.
    The product $\delta_i\mathbf{A}$ is simultaneously diagonalizable with $\mathbf{A}$ so
    \begin{align*}
        \barA_i &=\exp(\mathbf{W}\delta_i \diag(\mathbf{a})\mathbf{W}^{-1}) \tag{\Cref{lem:factor_scalar}} \\
        &= \mathbf{W}\exp(\diag(\mathbf{a}))\mathbf{W}^{-1}. \tag{\Cref{lem:factor_exp}}
    \end{align*}
    Since the exponential of scalars is $\L$-uniform $\TC^0$ computable by \Cref{cor:exp-log-softplus}, then $\barA_i$ is as well.

    Turning to $\barB_i$, note that the term $(\delta_i\mathbf{A})^{-1}$ is $\L$-uniform $\TC^0$ computable by \Cref{lem:invert_diagonalizable} since $\delta_i\mathbf{A}$ is diagonalizable. Since $\barA_i$ is $\L$-uniform $\TC^0$ computable, the difference $\barA_i - \mathbf{I}$ is as well. Then every term in 
    $$
    \barB_i = (\delta_i\mathbf{A})^{-1}(\barA_i - \mathbf{I})\delta_i\mathbf{B}_i
    $$
    is $\L$-uniform $\TC^0$ computable, and so their product is as well.
\end{proof}

\begin{remark}
    Since $\mathbf{C}_i$ and $\mathbf{D}_i$ are unchanged between the standard and diagonalizable versions of S6, the proofs of their computability as functions of $\mathbf{x}_i$ in $\L$-uniform $\TC^0$ pass through from \Cref{app:s6}.
\end{remark}

\begin{corollary}\label{cor:diagonalizable-s6}
    There exists an $\L$-uniform $\TC^0$ circuit family that computes Diagonalizable S6's convolutional form.
\end{corollary}

\begin{proof}
    Note that since $\mathbf{A} = \mathbf{W}\diag(\mathbf{a})\mathbf{W}^{-1}$ is fixed the set of transition matrices $\{\barA_i\}$ is simultaneously diagonalizable via $\mathbf{W}$ for all $i$.
    
    Then Diagonalizable S6 meets the conditions for \Cref{thm:diagonalizable}.
\end{proof}

\section{Nonlinearities in $\L$-Uniform $\TC^0$} \label{sec:nonlinearities}.

The parameterization of SSMs (and transformers) involves computing nonlinearities like $\exp$ and $\softplus$. We leverage existing circuit complexity results \citep{reif1992threshold} to show that, in general, any well-behaved nonlinearity should be computable in $\L$-uniform $\TC^0$ when used in conjunction with pre- or post-layer norm.

\begin{lemma}[Adapts Corollary 3.3, \citealp{reif1992threshold}] \label{lem:nonlinearity}
    Let $X = (-B, B)$ be a bounded interval.
    Let $f$ be a function over $X$ with a convergent Taylor series:
    \begin{equation*}
        f(x) = \sum_{n=0}^\infty \frac{a_n}{b_n} (x - x_0)^n ,
    \end{equation*}
    where $a_n, b_n$ are integers with magnitude at most $2^{n^{O(1)}}$ computable in $L$-uniform $\TC^0$.
    Then $f$ can be approximated over $X$ by $\L$-uniform $\TC^0$ circuits to log precision (error at most $2^{-n^c}$ for any $c \geq 1$).
\end{lemma}

\begin{proof}
    \citet{reif1992threshold} give a proof when $X = (-1, 1)$. We generalize to $X = (-B, B)$, assuming w.l.o.g. $B = 2^k$. The idea is to transform $f$ to have domain $(-1, 1)$ via
    \begin{equation*}
        g(x) = f(Bx) .
    \end{equation*}
    Then, we can apply Corollary 3.3 of \citet{reif1992threshold} to approximate $g$ with error at most $2^{-n^c}$. \citet{reif1992threshold} state their result for $\P$-uniform $\TC^0$, but through advances in circuit complexity since the time of publication~(\Cref{subsec:polynomial-division}), their construction naturally applies for $\L$-uniform $\TC^0$ as well.
    
    To approximate $f$, compute $z = x / B$, which can be done exactly since $B = 2^k$. We conclude by computing $g(z) = f(x)$, which, as established, has error at most $2^{-n^c}$.
\end{proof}


Because of pre- and post-norm layers, the elements of $\mathbf x_i$ in an SSM will remain in a bounded domain $(-B, B)$. Thus, the following lemma shows we can compute them:

\begin{corollary} \label{cor:exp-log-softplus}
    The pointwise nonlinearities $\exp$, $\log$, and $\softplus$ are computable over $(-B, B)$ in $\L$-uniform $\TC^0$.
\end{corollary}

\begin{proof}
    By \citet[Corollary 3.3]{reif1992threshold} know that the Taylor series for $\exp$ and $\log$ is convergent with $a_n,b_n$ computable in $\L$-uniform $\TC^0$. Then $\exp$ and $\log$ are themselves computable in $\L$-uniform $\TC^0$.

    Since $\softplus(x) = \log \left( 1 + \exp(x) \right)$ is a fixed composition of $L$-uniform $\TC^0$-computable functions, it too is computable in $L$-uniform $\TC^0$.
\end{proof}

\end{document}